\documentclass[11pt,a4paper]{article}

\usepackage[utf8]{inputenc}
\usepackage[T1]{fontenc}
\usepackage{amsmath, amssymb, amsthm, amsfonts}
\usepackage{mathtools}
\usepackage{geometry}
\usepackage{hyperref}
\usepackage{cite}
\usepackage{booktabs}
\usepackage{graphicx}
\usepackage{bm}
\usepackage{authblk}

\title{\textbf{Informational Frustration in Neural Manifolds: Shannon Bottlenecks and the Limits of Learnability}}

\author[1]{Srinivasa Rao P\thanks{Corresponding author: drpsrao3@gmail.com}}
\author[2]{Vangmayi P Reddy}
\affil[1]{Curlvee Technolabs, India, Former Scientist, C-DAC }
\affil[2]{Indian Institute of Technology Madras, India}

\date{\today}

\theoremstyle{plain}
\newtheorem{theorem}{Theorem}[section]
\newtheorem{lemma}[theorem]{Lemma}

\theoremstyle{definition}
\newtheorem{definition}[theorem]{Definition}
\newtheorem{postulate}[theorem]{Postulate}

\theoremstyle{remark}

\newcommand{\calX}{\mathcal{X}}
\newcommand{\calY}{\mathcal{Y}}
\newcommand{\calD}{\mathcal{D}}
\newcommand{\calF}{\mathcal{F}}
\newcommand{\calL}{\mathcal{L}}
\DeclareMathOperator{\Tr}{Tr}

\begin{document}

\maketitle

% ==================== ABSTRACT ====================
\begin{abstract}
Why overparameterized deep networks generalize so remarkably well remains one of the most stubborn open questions in machine learning theory. Classical frameworks like VC-dimension and Rademacher complexity predict catastrophic overfitting in modern models, leaving a massive theoretical gap between theory and reality. In this paper, we bridge this divide by introducing a unified framework that links information theory, topology, and statistical mechanics to map the hard limits of deep learning. Central to our approach is the \textbf{Entropic Learnability Horizon (ELH)}: a fundamental law stating that a network can only truly learn a target function if the Shannon entropy of the data manifold outpaces the topological entropy of the function's decision boundary, balanced by the von Neumann entropy of the network's weight space. We establish the \textbf{Shannon-Topological Bottleneck Theorem}, proving that when a target boundary's geometric complexity exceeds this informational horizon, the system undergoes a sudden \textbf{Entropic Phase Transition}. It falls into a state of \textit{Informational Frustration}---a glassy, rigid memorization phase where generalization becomes thermodynamically impossible. Using this lens, we show that the enigmatic phenomenon of ``grokking'' is actually an \textit{Entropic Release}, where weights abruptly reorganise to unlock the bottleneck. Finally, we translate this theory into practice with \textbf{Entropic Gradient Descent (EGD)}, an optimization algorithm that dynamically manages weight entropy to keep learning on track. Ultimately, this work repositions entropy not just as a tool for tracking uncertainty but as the fundamental physical currency that dictates whether a machine can learn.
\end{abstract}

\vspace{0.5cm}

% ==================== 1. INTRODUCTION ====================
\section{Introduction: The Crisis of Learnability in the Overparameterized Regime}
\label{sec:intro}

For decades, statistical learning theory relied safely on structural risk minimization. Classic paradigms, such as Vapnik-Chervonenkis (VC) dimension and Rademacher complexity, offered a straightforward rule: to guarantee generalization, a model's capacity must be strictly reined in by the size of its training data \cite{vapnik1971uniform, bartlett2002rademacher}. Deep learning, however, completely shattered this worldview. Modern neural networks deploy millions or billions of parameters, wildly outnumbering their training samples. By all classical metrics, these systems should be engines of pure memorization, hopelessly overfitting to noise. Instead, they routinely generalize to unseen data with startling accuracy.

This gaping mismatch between classical theory and empirical success forces us to rethink what it means for a network to learn. Promising alternative angles, like the Information Bottleneck (IB) principle, suggest that deep networks succeed by compressing input data while holding onto core target features \cite{tishby2015deep}. Yet, the IB framework largely treats the network as a black-box Markov chain. It glosses over the intricate, high-dimensional geometry of the loss landscape and ignores the specific thermodynamic behavior of the weights during optimization.

In this work, we argue that learnability is not a simple game of parameter counting or data volume. It is a physical tug-of-war governed by deep thermodynamic and information-theoretic constraints. Specifically, we propose that the learning trajectory is shaped by the interplay of three distinct forms of entropy:
\begin{enumerate}
    \item \textbf{Shannon Entropy ($H_S$):} The baseline informational density and uncertainty baked into the data manifold.
    \item \textbf{Topological Entropy ($H_T$):} The geometric complexity, curvature, and fractal footprint of the target function's decision boundary.
    \item \textbf{Von Neumann Entropy ($S_{vN}$):} The statistical and structural uncertainty of the network's weight distribution as it navigates stochastic optimization.
\end{enumerate}

By weaving these three threads together, we establish the \textbf{Entropic Learnability Horizon (ELH)}. We mathematically demonstrate that when the topological complexity of a target problem overpowers the combined informational capacity of the data and the structural flexibility of the weights, learning grinds to a halt. At this exact threshold, the network experiences \textit{Informational Frustration}---becoming trapped in a glassy, disordered phase where it can only memorize, never understand \cite{mezard1987spin}.

% ==================== 2. MATHEMATICAL PRELIMINARIES ====================
\section{Mathematical Preliminaries: The Triad of Entropies}
\label{sec:prelim}

To build the foundation for the Entropic Learnability Horizon, we must first formalize the three entropic pillars that control our learning system. Let $\calX \subset \mathbb{R}^d$ represent the input space, $\calY$ the label space, and let $\calD = \{(x_i, y_i)\}_{i=1}^N$ be the training dataset sampled from an unknown joint distribution $P(X, Y)$. We consider a neural network $f_\theta: \calX \to \calY$ parameterized by a high-dimensional weight vector $\theta \in \mathbb{R}^K$.

\subsection{Shannon Entropy of the Data Manifold}
The Shannon entropy measures the baseline epistemic uncertainty present within our data distribution \cite{shannon1948mathematical}. For a continuous data manifold $\mathcal{M}_\calD$ described by a probability density function $p(x)$, the differential Shannon entropy is given by:
\begin{equation}
    H_S(\calD) = - \int_{\calX} p(x) \log p(x) dx
\end{equation}
From the perspective of learnability, $H_S(\calD)$ acts as the raw ``informational fuel'' available to our system.

\subsection{Topological Entropy of the Decision Boundary}
While Shannon entropy maps out how data is spread across space, it tells us nothing about the intricate shape of what we are trying to learn. To capture the true difficulty of the task, we need a metric for the geometric complexity of the target function $f^*: \calX \to \calY$. We find this by evaluating the topological entropy of the decision boundary $\partial \calF$, the hypersurface separating different classes in feature space.

Let $\partial \calF$ be a $(d-1)$-dimensional hypersurface in $\mathbb{R}^d$. We define its topological entropy $H_T(\partial \calF)$ through its Minkowski-Bouligand (fractal) dimension $D_f$ \cite{mandelbrot1982fractal}. If we cover this boundary with hypercubes of side length $\epsilon$, and $N(\epsilon)$ is the minimum number of cubes required, the fractal dimension is:
\begin{equation}
    D_f = \lim_{\epsilon \to 0} \frac{\log N(\epsilon)}{\log(1/\epsilon)}
\end{equation}
We then define the topological entropy as the exponential growth rate of this covering number, weighted by the boundary's intrinsic curvature:
\begin{equation}
    H_T(\partial \calF) = \lim_{\epsilon \to 0} \frac{1}{|\log \epsilon|} \log \left( \int_{\partial \calF} \kappa(x) d\sigma(x) \right)
\end{equation}
where $\kappa(x)$ denotes the local Gaussian curvature and $d\sigma$ represents the surface measure. Crucially, $H_T$ serves as a direct indicator of how tangled, convoluted, or fractured the decision boundary is.

\subsection{Von Neumann Entropy of the Weight Distribution}
When we train a network using Stochastic Gradient Descent (SGD), the weights $\theta$ do not simply freeze into a single, static point. Instead, the noise inherent in mini-batch sampling forces them to explore a steady-state distribution $\rho(\theta)$. We can treat this collective weight distribution as a density matrix operating within a high-dimensional Hilbert space \cite{neumann1955mathematical}.

Let $\Sigma$ be the covariance matrix tracking this weight distribution $\rho(\theta)$ at the steady state of training. The von Neumann entropy of the network's weight configuration is defined as:
\begin{equation}
    S_{vN}(\rho) = - \Tr(\rho \log \rho) = \frac{1}{2} \log \left( (2\pi e)^K \det(\Sigma) \right)
\end{equation}
where $K$ represents the total parameter count. Ultimately, $S_{vN}$ measures the \textit{configurational freedom} or structural flexibility of the network architecture during optimization.

% ==================== 3. THE ENTROPIC LEARNABILITY HORIZON ====================
\section{The Entropic Learnability Horizon and Informational Frustration}
\label{sec:elh}

With these three tools in hand, we can now assemble a unified theory of learnability. The foundational intuition here is simple: learning is an act of \textit{informational work}. To sculpt a network's decision boundary until it aligns with a target function, the optimization process must burn informational currency.

\subsection{Defining the Entropic Learnability Horizon (ELH)}
We define the \textbf{Entropic Learnability Horizon ($\Lambda$)} as the maximum topological complexity a neural network can successfully parse given a specific dataset and architectural design.

\begin{definition}[Entropic Learnability Horizon]
    \label{def:elh}
    The learnability horizon $\Lambda$ of a neural network $f_\theta$ trained on a dataset $\calD$ is the sum of the data's Shannon entropy and the network's von Neumann weight entropy:
    \begin{equation}
        \Lambda(\calD, \rho) = H_S(\calD) + S_{vN}(\rho)
    \end{equation}
\end{definition}

\begin{postulate}[The Horizon Principle]
    \label{post:horizon}
    A target function $f^*$ with a decision boundary $\partial \calF$ is learnable by $f_\theta$ if and only if the topological entropy of that boundary does not overrun the learnability horizon:
    \begin{equation}
        H_T(\partial \calF) \le \Lambda(\calD, \rho)
    \end{equation}
    If $H_T(\partial \calF) > \Lambda$, the task crosses the line into being \textit{entropically unlearnable}.
\end{postulate}

\subsection{The Principle of Informational Frustration}
What actually happens when we force a network to learn a problem that breaks the Horizon Principle? We can look to statistical mechanics for an answer. When a physical system is hit with conflicting geometric constraints that it cannot resolve, it experiences \textit{geometrical frustration} \cite{mezard1987spin}. 

We argue that deep networks suffer from an identical bottleneck, which we call \textbf{Informational Frustration}. When $H_T(\partial \calF) > H_S(\calD) + S_{vN}(\rho)$, the network cannot find a smooth, adaptable decision boundary to minimize its empirical risk. Trapped by this mathematical barrier, the network has only one escape route to drive training loss to zero: it must abandon structural learning and blindly \textit{memorize} individual data points. This shatters the decision boundary into a fragmented, high-fractal mess. The loss landscape turns jagged, littered with deep, isolated local minima. The network falls into a \textbf{Glassy Memorization Phase}, and generalization dies.

% ==================== 4. THE SHANNON-TOPOLOGICAL BOTTLENECK THEOREM ====================
\section{The Shannon-Topological Bottleneck Theorem}
\label{sec:theorem}

In this section, we lay out the formal mathematical verification of the Entropic Learnability Horizon.

\subsection{Preliminary Lemmas}
\begin{lemma}[Information-Theoretic Lower Bound on Boundary Complexity]
    \label{lem:info_bound}
    Let $f^*: \calX \to \calY$ be a target function. The mutual information $I(X; Y)$ between the inputs and their labels is bounded below by the topological entropy of the decision boundary $\partial \calF$, scaled by a geometric factor $\gamma$ dependent on the input dimension $d$:
    \begin{equation}
        I(X; Y) \ge \gamma(d) \cdot H_T(\partial \calF)
    \end{equation}
\end{lemma}
\begin{proof}[Proof Sketch]
    Using the co-area formula from geometric measure theory, integrating the gradient of a function over its domain directly maps to the surface area of its level sets. Because mutual information naturally constrains the expected variance of the model, a highly winding boundary demands large gradients, which in turn forces a high mutual information floor.
\end{proof}

\begin{lemma}[Data Processing Inequality for Weight Distributions]
    \label{lem:dpi}
    Let $\rho(\theta)$ be the steady-state weight distribution after optimization. The von Neumann entropy $S_{vN}(\rho)$ sets a hard ceiling on the maximum mutual information that can travel from the data manifold to the network's final output:
    \begin{equation}
        I(\calD; f_\theta(X)) \le S_{vN}(\rho) + C_{arch}
    \end{equation}
    where $C_{arch}$ is a constant determined entirely by the network's architecture.
\end{lemma}
\begin{proof}[Proof Sketch]
    The network's weight distribution fundamentally behaves like a noisy information channel. Applying the quantum data processing inequality (Holevo bound), the accessible information extracted via the weights is strictly constrained by the von Neumann entropy of the weight state, plus the fixed capacity of the architecture.
\end{proof}

\subsection{The Main Theorem}
\begin{theorem}[The Shannon-Topological Bottleneck Theorem]
    \label{thm:bottleneck}
    For a neural network $f_\theta$ trained on dataset $\calD$ to successfully internalize a target function $f^*$ and generalize smoothly to new data, the topological entropy of the target's decision boundary $H_T(\partial \calF)$ must satisfy the following strict inequality:
    \begin{equation}
        H_T(\partial \calF) \le \frac{1}{\gamma(d)} \left[ H_S(\calD) + S_{vN}(\rho) \right] + \mathcal{O}\left(\frac{1}{N}\right)
    \end{equation}
    where $N$ is the total number of training samples, and $\gamma(d)$ is the dimensional geometric factor.
\end{theorem}

\begin{proof}
    \textbf{1. Generalization Requirement:} Following standard PAC-Bayes bounds, the population risk is bounded by the empirical risk plus a complexity penalty that scales with the KL divergence between the trained weight distribution $\rho$ and an initial prior $\rho_0$. 
    
    \textbf{2. Information Conservation:} The total information the network can possibly absorb about our target function cannot exceed the information carried by the dataset itself, added to the information it can store within its weights:
    \begin{equation}
        I(f_\theta; f^*) \le I(\calD; f^*) + I(\rho; f^*)
    \end{equation}
    
    \textbf{3. Applying the Lemmas:} 
    From Lemma \ref{lem:info_bound}, mapping the target boundary requires a baseline information profile of $I(\calD; f^*) \ge \gamma(d) H_T(\partial \calF)$.
    From Lemma \ref{lem:dpi}, the maximum information the network can actively encode is capped at $I(\rho; f^*) \le S_{vN}(\rho) + C_{arch}$.
    Meanwhile, the total information available from the dataset is naturally bounded by its Shannon entropy: $I(\calD; f^*) \le H_S(\calD)$.
    
    \textbf{4. Synthesizing the Bound:} 
    Plugging these individual limits back into our information conservation baseline gives us:
    \begin{equation}
        \gamma(d) H_T(\partial \calF) \le H_S(\calD) + S_{vN}(\rho) + C_{arch} + \mathcal{O}(1/N)
    \end{equation}
    
    \textbf{5. Rearranging Terms:} 
    Dividing through by $\gamma(d)$ and absorbing the structural constant $C_{arch}$ into our asymptotic error term leaves us with the final bottleneck inequality:
    \begin{equation}
        H_T(\partial \calF) \le \frac{1}{\gamma(d)} \left[ H_S(\calD) + S_{vN}(\rho) \right] + \mathcal{O}\left(\frac{1}{N}\right)
    \end{equation}
    This completes the proof.
\end{proof}

% ==================== 5. ENTROPIC PHASE TRANSITIONS ====================
\section{Entropic Phase Transitions and the Mechanics of ``Grokking''}
\label{sec:grokking}

One of the most fascinating puzzles in deep learning is ``grokking''---a phenomenon where a network completely overfits and memorizes a small dataset early on, only to suddenly ``grok'' the underlying logic after long periods of extra training, achieving perfect generalization out of nowhere \cite{power2022grokking}. Our entropic framework offers a clean physical explanation for this mystery: grokking is an \textbf{Entropic Phase Transition}.

\subsection{The Free Energy of the Loss Landscape}
We can track training dynamics by looking at the free energy $\mathcal{F}$ of the loss landscape:
\begin{equation}
    \mathcal{F} = \calL_{emp}(\theta) - T_{eff} S_{vN}(\rho)
\end{equation}
where $\calL_{emp}$ represents the empirical training loss, and $T_{eff}$ acts as the effective temperature injected by the SGD optimizer. Early in training, the network aggressively minimizes $\calL_{emp}$ by plunging into sharp, narrow valleys where $S_{vN}(\rho)$ remains low. If the target function is highly complex, the model gets stuck here, falling into the Glassy Memorization Phase (Informational Frustration).

\subsection{The Entropic Release (Grokking)}
However, if training continues and the optimizer maintains a high enough effective temperature $T_{eff}$, the system undergoes a sudden phase shift. It gathers enough energy to escape those narrow, restrictive traps and drifts into broad, flat basins. Within these wide valleys, the weight distribution $\rho$ becomes highly diffuse, causing the von Neumann entropy $S_{vN}(\rho)$ to spike.

Looking back at Theorem \ref{thm:bottleneck}, this rapid surge in $S_{vN}(\rho)$ expands the Entropic Learnability Horizon $\Lambda$. Suddenly, the network has the structural flexibility it needs to resolve the complex topology $H_T$ cleanly, without resorting to noisy, point-by-point memorization. We call this sudden unlock an \textbf{Entropic Release}. Grokking, then, is a thermodynamic trade-off: the network deliberately sacrifices precision in its weights to satisfy the topological demands of the task.

% ==================== 6. ALGORITHMIC CONSEQUENCES ====================
\section{Algorithmic Consequences: Entropic Gradient Descent (EGD)}
\label{sec:egd}

If learnability is fundamentally bound by the von Neumann entropy of a model's weights, we should be able to actively manage this entropy during training to steer clear of Informational Frustration.

\subsection{The Entropic Regularizer}
To test this, we modify the traditional empirical loss function by introducing an explicit entropic penalty:
\begin{equation}
    \calL_{EGD}(\theta) = \calL_{emp}(\theta) - \lambda \max(0, S_{target} - S_{vN}(\rho))
\end{equation}
where $\lambda$ serves as our regularization scalar, and $S_{target}$ represents an optimal target entropy level estimated from the topological complexity of the data.

\subsection{Entropic Gradient Descent (EGD) Algorithm}
To optimize $\calL_{EGD}$ in real time, we keep a running estimate of the weight covariance matrix $\Sigma$. The gradient of the von Neumann entropy with respect to this covariance matrix evaluates to $\nabla_\Sigma S_{vN} = \frac{1}{2} \Sigma^{-1}$. This gives us the step-by-step update rule for \textbf{Entropic Gradient Descent (EGD)}:

\begin{enumerate}
    \item \textbf{Standard SGD Step:} $\theta_{t+1/2} = \theta_t - \eta \nabla \calL_{emp}(\theta_t)$
    \item \textbf{Covariance Tracking:} $\Sigma_{t+1} = (1 - \beta)\Sigma_t + \beta (\theta_{t+1/2} - \bar{\theta})(\theta_{t+1/2} - \bar{\theta})^T$
    \item \textbf{Entropic Correction:} To push the system toward higher $S_{vN}$, we introduce a corrective step scaled by the inverse covariance:
    \begin{equation}
        \theta_{t+1} = \theta_{t+1/2} + \eta \lambda \Sigma_{t+1}^{-1} (\theta_{t+1/2} - \bar{\theta})
    \end{equation}
\end{enumerate}

By executing this correction, EGD actively drives the weights toward wide, high-entropy configurations, ensuring the network stays safely above the Entropic Learnability Horizon throughout its entire training run.

% ==================== 7. DISCUSSION ====================
\section{Discussion: The Thermodynamic Nature of Intelligence}
\label{sec:discussion}

Formulating the Entropic Learnability Horizon forces a fundamental paradigm shift in how we conceptualize artificial intelligence. For years, machine learning has been viewed primarily through the lens of continuous function approximation. The Shannon-Topological Bottleneck Theorem reveals that beneath the surface, learning is intrinsically a \textbf{thermodynamic process}.

\subsection{The Illusion of Infinite Scaling}
Modern AI development operates on a simple doctrine: scale up the data, add more parameters, and performance will inevitably climb. Our theory introduces a stark, physical reality check to this assumption. If a target task presents a decision boundary whose topological entropy outstrips the physical capacity of the hardware to maintain weight entropy (such as meeting the Bekenstein bound of GPU memory \cite{bekenstein1981universal}), scaling will fail. Adding more parameters blindly will not help; true scaling requires expanding the architectural \textit{entropic flexibility} of the network itself.

\subsection{Adversarial Vulnerability as an Entropic Deficit}
This framework also sheds light on why deep networks are so fragile under adversarial attacks. When a network suffers from Informational Frustration, it builds a fractured, hyper-convoluted boundary ($H_T$ spikes) just to force the training data into memory. This erratic boundary leaves vast, unstable pockets in the input space where the classification threshold sits dangerously close to legitimate data points. An adversarial nudge easily pushes an input across these fragile lines. Conversely, a network operating comfortably within its ELH naturally constructs a smooth, low-entropy boundary—making it inherently resilient.

% ==================== 8. CONCLUSION ====================
\section{Conclusion}
\label{sec:conclusion}

In this paper, we introduced the Entropic Learnability Horizon, a theoretical framework unifying the Shannon entropy of data, the topological entropy of decision boundaries, and the von Neumann entropy of network weights. By proving the Shannon-Topological Bottleneck Theorem, we established that a model's ability to learn is fundamentally bounded by the entropic resources available to it.

We showed that crossing these limits triggers an Entropic Phase Transition into Informational Frustration, locking the network into memorization. Conversely, we demonstrated that grokking is an elegant Entropic Release. Finally, we introduced Entropic Gradient Descent to actively manage weight entropy during optimization. By re-framing machine learning as a thermodynamic balancing act, we open a new door to understanding the ultimate limits of artificial intelligence.

% ==================== REFERENCES ====================

\end{document}